\documentclass[twocolumn]{autart}    

\usepackage{times}
\usepackage{xcolor}
\usepackage[pdftex]{graphicx} 
\usepackage{algorithm}
\usepackage{algpseudocode}

\usepackage{enumitem}
\usepackage{MnSymbol}%
\usepackage{wasysym}
\usepackage[numbers]{natbib}
\usepackage{url}
\usepackage[normalem]{ulem}
\usepackage{wrapfig}

\newcommand{\robots}{\mathcal{R}}
\newcommand{\Cells}{\mathcal{V}}
\newcommand{\Edges}{\mathcal{E}}
\newcommand{\Bottles}{\mathcal{B}}
\newcommand{\Session}{\mathcal{S}}
\newcommand{\cell}{v}
\newcommand{\curr}[1]{curr(r_{#1})}
\newcommand{\bottle}[3]{{b}_{#2, #3}^{#1}}

\newcommand{\traj}[2]{\pi_{#1}^{#2}}
\newcommand{\graphname}{Path-Graph}
\newcommand{\Statexx}{\Statex \hskip1em}
\newcommand{\Statexxx}{\Statex \hskip2em}

\newtheorem{example}{Example}
\newtheorem{remark}{Remark}
\newtheorem{Problem}{Problem}
\newtheorem{theorem}{Theorem}

\newtheorem{definition}{Definition}
\newenvironment{proof}{\paragraph*{Proof:}}{\hfill$\square$}

\algdef{SE}[SUBALG]{Indent}{EndIndent}{}{\algorithmicend\ }%
\algtext*{Indent}
\algtext*{EndIndent}

\newcommand{\remove}[1]{}

\begin{document}

\begin{frontmatter}

\title{From Drinking Philosophers to Asynchronous Path-Following Robots}


\author{Yunus Emre Sahin}\ead{ysahin@umich.edu},    
\author{Necmiye Ozay}\ead{necmiye@umich.edu}               

\address{Department of Electrical Engineering and Computer Science, University of Michigan, Ann Arbor, MI}  

\begin{keyword}                           
Multi-agent systems;  Autonomous mobile robots; Concurrent systems; Deadlock avoidance. 
\end{keyword}                             

\begin{abstract}

In this paper, we consider the multi-robot path execution problem where a group of robots move on predefined paths from their initial to target positions while avoiding collisions and deadlocks in the face of asynchrony. We first show that this problem can be reformulated as a distributed resource allocation problem and, in particular, as an instance of the well-known Drinking Philosophers Problem (DrPP). By careful construction of the drinking sessions capturing shared resources, we show that any existing solutions to DrPP can be used to design robot control policies that are collectively collision and deadlock-free. We then propose modifications to an existing DrPP algorithm to allow more concurrent behavior, and provide conditions under which our method is deadlock-free. Our method does not require robots to know or to estimate the speed profiles of other robots and results in distributed control policies. We demonstrate the efficacy of our method on simulation examples, which show competitive performance against the state-of-the-art.
\end{abstract}

\end{frontmatter}

\section{Introduction}\label{chp:intro}
Multi-robot path planning (MRPP) has been one of the fundamental problems studied by artificial intelligence and robotics communities. Quickly finding paths that take each robot from their initial location to target location, and ensuring that robots execute these paths in a safe manner have applications in many areas from evacuation planning \cite{luh2012modeling} to warehouse robotics \cite{wurman2008coordinating}, and from formation control \cite{tanner2004leader} to coverage \cite{mcnew2007solving}.

There are several challenges in multi-robot path planning such as scalability, optimality, trading off centralized versus distributed decisions and corresponding communication loads, and potential asynchrony. Planning optimal collision-free paths is known to be hard \cite{surynek2010optimization} even in synchronous centralized settings. Recently developed heuristics aim to address the scalability challenge when optimality is a concern \cite{yu2016optimal}. Arguably, the problem gets even harder when there is non-determinism in the robot motions. One source of non-determinism is asynchrony, that is, the robots can move on their individual paths with different and time-varying speeds and their speed profiles are not known a priori. This might happen due to many factors such as low battery levels, calibration errors, or waiting to give way to humans in the workspace. The goal of this paper is, given a collection of paths, one for each robot, to devise a distributed protocol so that the robots are guaranteed to reach their targets in the face of asynchrony and avoid all collisions along the way. We call this the \emph{Multi-Robot Path Execution (MRPE)}  problem. We study this problem at a discrete-level where the execution policies we generate can be used with any suitably abstracted continuous-dynamics and low-level control policies that allow stopping the robot when needed. Moreover, our execution policies can interface with higher-level decision-making modules that lead to asynchrony by temporarily stopping the individual robots for emergencies, maintenance, or other reasons. {In addition to MRPE, our execution policies can also be relevant for other applications such as transportation networks where vehicles travel on fixed tracks or manufacturing processes where workpieces follow complex conveyor networks.}

The key insight of the paper is to recast the MRPE problem as a drinking philosophers problem (DrPP) \cite{chandy-misra}, an extension of the well-known dining philosophers problem \cite{dining}. DrPP is a resource allocation problem for distributed and concurrent systems. By partitioning the workspace into a set of discrete cells and treating each cell as a shared resource, we show how to construct drinking sessions such that the MRPE problem can be solved using any DrPP algorithm. Existing DrPP algorithms, such as \cite{chandy-misra, ginat1989efficient, welch1993modular}, can be implemented in distributed manner, and enjoy nice properties such as fairness and deadlock-freeness, while also guaranteeing collision avoidance when applied to multi-robot planning. To allow more concurrent behavior and to improve the overall performance, we further modify \cite{ginat1989efficient}, and derive conditions on the collection of paths such that collisions and deadlocks are guaranteed to be avoided.

The rest of the paper is organized as follows. Section~\ref{sec:related} briefly presents related work. Section~\ref{chp:prelim} formally defines the MRPE problem we are interested in solving. A brief summary of the DrPP and an existing solution is presented in Section~\ref{chp:drinking}. Section~\ref{chp:map} recasts the MRPE problem as a DrPP, and shows that existing methods can be used to solve MRPE problems. Furthermore, this section provides modifications to \cite{ginat1989efficient} that allow more concurrent behavior. Section~\ref{chp:results} shows that, when fed by the same paths, our algorithm achieves competitive results with the state-of-the-art \cite{ma2017multi}. Section~\ref{chp:conclusion} concludes the paper.

\section{Related Work}\label{sec:related}
{Multi-robot path planning deals with the problem of planning a collection of paths that take a set of robots from their initial position to a goal location without collisions. In this paper, we consider the type of problems where the workspace is partitioned into a set of discrete cells, each of which can hold at most one robot, and time is discretized. Most of the research in this field have been focused on finding optimal or suboptimal paths that minimize either the makespan (last arrival time) or the flowtime (sum of all arrival times) \cite{felner2017search,stern2019multi, yu2016optimal}. These methods require the duration of each actions to be fixed to show optimality. In real-life, however, robots cannot execute their paths perfectly. They might move slower or faster than intended due to various factors, such as low battery levels, calibration errors and other failures. Methods that deal with such uncertainties, which might lead to collisions or deadlocks if not handled properly, can be divided into two main groups. 
	
In the first group, robots are allowed to replan their paths at run-time \cite{panagou2014motion, csenbacslar2019robust, van2011reciprocal}. 
In this case, simpler path planning algorithms can be used, leaving the burden of collision avoidance to the run-time controllers. However, this approach might lead to deadlocks in densely crowded environments. Moreover, when the specifications are complex, changing paths might even lead to violations of the specifications. Therefore, replanning paths on run-time is not always feasible. Alternatively, collisions and deadlocks can be avoided without needing to replan on run-time \cite{desai2017drona, ma2017multi, reveliotis2008conflict, roszkowska2013distributed, sahin2018multirobot, zhou2018distributed,zhou2020distributed}. For instance, if the synchronization errors can be bounded, \cite{desai2017drona} and  \cite{sahin2018multirobot} show how to synthesize paths that are collision and deadlock-free. This is achieved by overestimating the positions of robots and treating them as moving obstacles. However, this is a conservative approach as the burden of collision and deadlock avoidance is moved to the offline planning part.

Alternatively, an execution policy can be used to decide when to move, slow down or stop robots on their predefined paths to avoid collisions and deadlocks. Given a collection of paths, which are collision and deadlock free under perfect synchronization, \cite{honig2016multi} provides a method for robust execution under synchronization errors. This work focuses on how to compute the temporal dependencies between robots and enforce a fixed ordering between each robot pair for all possible conflicts such that collisions and deadlocks are avoided. Same idea is used in \cite{ma2017multi}, named Minimal Communication Policies (MCP), but the focus is on the planning phase. A delay probability for each agent is used in the planning phase to optimize the expected makespan. This method is later expanded to lifelong missions in warehouse environments \cite{honig2019persistent}. Such a fixed ordering prevents collisions and deadlocks, however, it is limiting as the performance of the multirobot system depends highly on the exact ordering. If one of the robots experiences a failure at run-time and starts moving slowly, it might become the bottleneck of the whole system. In fact, we demonstrate the effects of such a scenario on the system performance and provide numerical results that show the robustness of our method.

When the collection of paths are known a priori, one can also find all possible collision and deadlock configurations, and prevent the system from reaching those. For instance, distributed methods in \cite{zhou2018distributed} and \cite{zhou2020distributed} find deadlock configurations by abstracting robot paths into an edge-colored directed graph. However, this abstraction step might be conservative. Imagine a long passage which is not wide enough to fit more than one robot, and two robots crossing this passage in the same direction. The entire passage would be abstracted as a single node, and even though robots can enter the passage at the same time and follow each other safely, they would not be allowed to do so. Instead, robots have to wait for the other to clear the entire passage before entering. Moreover, \cite{zhou2020distributed} require that no two nodes in the graph are connected by two or more different colored edges. This strong restriction limits the method's applicability to classical multirobot path execution problems where robots move on a graph and the same two nodes might be connected with multiple edges in each direction.

As connectivity and autonomous capabilities of vehicles improve, cooperative intersection management problems draw significant attention from researchers \cite{ahn2017safety, carlino2013auction, dresner2008multiagent, zohdy2016intersection}. These problems are similar to MRPE problem as both require coordinating multiple vehicles to prevent collisions and deadlocks. Compared to traditional traffic light-based methods, cooperative intersection management methods offer improved safety, increased traffic flow and lower emissions. We refer the reader to \cite{chen2015cooperative} for a recent survey on this topic and main solution approaches. Although they seem similar, the setting of intersection management problems are tailored specifically for the existing road networks, and thus, cannot be easily generalized to MRPE problems where robots/vehicles might be moving in non-structured environments.

Our method is based on reformulating the MRPE problem as a resource allocation problem. There are similar methods such as \cite{reveliotis2008conflict}, which requires a centralized controller, and \cite{roszkowska2013distributed}, which needs cells to be large enough to allow collision-free travel of up to two vehicles, instead of only one. We base our method on the well-known drinking philosopher algorithm. We show that any existing DrPP solution can be used to solve the MRPE problem if drinking sessions are constructed carefully. However, such methods require strong conditions on a collection of paths to hold, and limit the amount of concurrency in the system. To relax the conditions and to improve the performance, we provide a novel approach by taking the special structure of MRPE problems into account. We show that our method is less conservative than the naive approach, and provide numerical results to confirm the theoretical findings. Our approach leads to control policies that can be deployed in a distributed form.

\section{Problem Definition}\label{chp:prelim}
We start by providing definitions that are used in the rest of the paper and formally state the problem we are interested in solving. Let a set $\robots = \{r_1, \dots, r_{N}\}$ of robots share a workspace that is partitioned into set $\Cells$ of discrete cells. Two robots are said to be in \emph{collision} if they occupy the same cell at the same time. An ordered sequence $\pi = \{\pi^0, \pi^1, \dots\}$ of cells, where each $\pi^t\in \Cells$, is called a \emph{path}. The path segment $ \{\pi^t, \pi^{t+1}, \dots, \pi^{t'}\}$ is denoted by $\pi^{t:t'}$, and we write $\cell \in \pi$ if there exists a $t$ such that $\pi^t = v$. We assume that a path is given for each robot, and $\traj{n}{}$ denotes the path associated with $r_n$. We allow paths to contain loops but require them to be finite. We use $\traj{n}{end}$ and $\curr{n}$ to denote the final cell of $\pi_n$ and the number of successful transitions completed by $r_n$, respectively. We also define $next(r_n) \doteq \curr{n}+1$. The motion of each robot is governed by a \emph{control policy}, which issues one of the two commands at every time step: $(1)$ {$STOP$} and $(2)$ \emph{$GO$}. The $STOP$ action forces a robot to stay in its current cell. If the $GO$ action is chosen, the robot starts moving. This robot might or might not reach to the next cell within one time step, however, we assume that a robot eventually progresses if $GO$ action is chosen constantly. This non-determinism models the uncertainties in the environment, such as battery levels or noisy sensors/actuators, which might lead to robots moving faster or slower than intended. It can also be due to a higher-level decision maker that forces the robot to stay put for instance as an emergency stop to give way to humans in the workspace. We now formally define the  problem we are interested in solving:

\begin{Problem}\label{prob1}
	Given a collection $\Pi = \{\traj{1}{}, \dots \traj{N}{}\}$ of paths, design a control policy for each robot such that all robots eventually reach their final cells while avoiding collisions.
\end{Problem}

As stated in the problem definition, we assume that predefined paths are known by every robot prior to the start of execution. There are many control policies that can solve Problem~\ref{prob1}. For the sake of performance, policies that allow more concurrent behavior are preferred. In the literature, two metrics are commonly used to measure the performance: \emph{makespan (latest arrival time)} and \emph{flowtime (sum of arrival times)}.
 Given a set of robots $\robots = \{r_1, \dots, r_{N}\}$, if robot $r_n$ takes $t_n$ time steps to reach its final state, makespan and flowtime values are given by $\max_{1\leq n\leq N} t_n$ and $\sum_{n=1}^N t_n$, respectively. These values decrease as the amount of concurrency increases. However, it might not be possible to minimize both makespan and flowtime at the same time, and choice of policy might depend on the application. 

To solve Problem~\ref{prob1}, we propose a method that is based on the well-known drinking philosophers problem introduced by \cite{chandy-misra}. For the sake of completeness, this problem is explained briefly in Section \ref{chp:drinking}.

\section{Drinking philosophers problem}\label{chp:drinking}
The drinking philosophers problem is a generalization of the well-known \emph{dining philosophers problem} proposed by \cite{dining}. These problems capture the essence of conflict resolution, where multiple resources must be allocated to multiple processes. Given a set of processes and a set of resources, it is assumed that each resource can be used by at most one process at any given time. In our setting, processes and resources correspond to robots and discrete cells that partition the workspace, respectively. Similar to mutually exclusive use of the resources, any given cell can be occupied by at most one robot to avoid collisions. In the DrPP setting, processes are called \emph{philosophers}, and shared resources are called \emph{bottles}. A philosopher can be in one of the three \emph{states}: $(1)$ \emph{tranquil}, $(2)$ \emph{thirsty}, or $(3)$ \emph{drinking}. A \emph{tranquil} philosopher may stay in this state for an arbitrary period of time or \emph{become thirsty} at any time it wishes. A thirsty philosopher \emph{needs} a non-empty subset of bottles to drink from. This subset, called \emph{drinking session}, is not necessarily fixed, and it could change over time. After acquiring all the bottles in its current drinking session, a thirsty philosopher \emph{starts drinking}. After a finite time, when it no longer needs any bottles, the philosopher goes back to tranquil state. The goal of the designer is to find a set of rules for each philosopher for acquiring and releasing bottles. A desired solution would have the following properties. 
(i) \emph{Liveness:} A thirsty philosopher eventually starts drinking. In our setting liveness implies that each robot is eventually allowed to move. 
(ii) \emph{Fairness:} No philosopher is consistently favored over another. In multi-robot setting, fairness indicate that there is no fixed priority order between robots.
(iii) \emph{Concurrency:} Any pair of philosophers must be allowed to drink at the same time, as long as they wish to drink from different bottles. Analogously, no robot waits unnecessarily.

We base our method on the DrPP solution of \cite{ginat1989efficient}, which is shown in Algorithm~\ref{alg:base}. This solution ensures liveness, fairness and concurrency. For the sake of completeness, we provide a brief summary of their solution, but refer the reader to \cite{ginat1989efficient} for the proof of correctness and additional details.

Each philosopher $p$ has a unique integer $id_p$ and keeps track of two non-decreasing integers: session number $s\_num_p$ and the highest received session number $max\_rec_p$. These integers are used to keep a strict priority order between the philosophers. Conflicts are resolved according to this order, in favor of the philosopher with the higher priority. We say that \emph{philosopher $p$ has higher priority than philosopher} $r$ (denoted  \emph{$p\prec r$}) if $(s\_num_p, id_p) \prec (s\_num_r, id_r)$ that is $s\_num_p < s\_num_r$, or $s\_num_p = s\_num_r$ and $id_p<id_r$. That is, smaller session number indicates a higher priority, and in the case of identical session numbers, philosopher with the smaller $id$ has the higher priority. Note that the priority order is not fixed as $s\_num$ values change over time.

Philosophers also keep track of several Boolean variables. Let $p$ and $r$ be two philosophers and $b$ be a bottle shared between them. It must be noted that, philosophers $p$ and $r$ can share multiple bottles among each other, but each bottle $b$ is shared by exactly two philosophers. For each bottle $b$ in philosopher $p$'s inventory, we define two booleans $hold_p(b)$ and $req_p(b)$, and say that $p$ holds the bottle $b$ (or the request token for the bottle $b$) if the variable $hold_p(b)$ (or $req_p(b)$) holds $true$. 
Similarly, we say that $p$ needs $b$ if $need_p(b)$ hold $true$. Algorithm~\ref{alg:base} ensures mutual exclusiveness, that is, $hold_p(b)$ and $hold_r(b)$ (similarly $req_p(b)$ and $req_r(b)$) cannot be $true$ at the same time.

Philosophers are initialized such that each philosopher is in {tranquil} state and $s\_num$ and $max\_rec$ are set to $0$. Bottles and associated request tokens are shared between philosophers such that one philosopher holds the bottle while the other holds the associated request token. Since philosophers are in tranquil state, all $need(b)$ variables are  initially $false$.

The rules $R1-R6$ of Algorithm~\ref{alg:base} are triggered by events. For example, if a philosopher $p$ wants to drink from a set of bottles $\Session$, it needs to \emph{become thirsty} first by executing $R1$. Upon holding all the bottles in $\Session$, $R2$ is triggered and $p$ starts drinking.
Similarly, $p$ triggers $R4$ when 
$p$ (i) needs $b$, (ii) does not currently hold $b$, and (iii) holds the associated request token $req_p(b)$. Then, $p$ requests the bottle from $r$ by sending the message $(req_b, s\_num_p, id_p)$. Receiving such a message, triggers $R5$ in $r$. If $r$ (1) does not need $b$ or (2) is thirsty and $(s\_num_p, id_p) \prec (s\_num_r, id_r)$, then it sets $hold_r(b)$ $false$ and sends the bottle to $p$. Sending a bottle is simply done by sending a message to $p$. Upon receiving such a message, $p$ sets $hold_p(b)$ $true$ by running $R6$. We refer the reader to \cite{ginat1989efficient} for more details.

 \begin{algorithm}[t]
	\caption{Drinking Philosopher Algorithm by \cite{ginat1989efficient}}\label{alg:base}
	\begin{algorithmic}[1]
		\Statex \hskip-1em {\bf R1: \emph{becoming\_thirsty} with session $\mathcal{S}$}
			\Statex {\bf for} each bottle $b\in \mathcal{S}$ {\bf do} 
			\Statexx $need_p(b) \gets true$
			\Statex {\bf end for}
			\Statex $s\_num_p \gets max\_rec_p + 1$
			\Statex
		\hskip-1.5em \hrulefill
		
		\Statex \hskip-1em {\bf R2: start \emph{drinking}}
		\Statex {\bf if} holding all needed bottles {\bf then}
			\Statexx become $drinking$
		\Statex {\bf end if}
					\Statex
					
		\hskip-1.5em \hrulefill
		\Statex\hskip-1em  {\bf R3: \emph{becoming\_tranquil}, honoring deferred requests}
		\Statex {\bf for} each consumed bottle $b$ {\bf do}
		\Statexx $need_p(b)\gets false;$
		\Statexx {\bf if} $req_p(b)$ {\bf then}
		\Statexxx $[hold_p(b) \gets false; Send(b)]$
		\Statexx {\bf end if}
		\Statex {\bf end for}
					\Statex
					
		\hskip-1.5em \hrulefill
		\Statex\hskip-1em  {\bf R4: \emph{requesting a bottle}}
		\Statex {\bf if} ($need_p(b)$ {\bf and} $\neg hold_p(b)$ {\bf and} $req_p(b))$) {\bf then} 
		\Statexx $req_p(b)\gets false$
		\Statexx $Send(req_b, s\_num_p, id_p)$
		\Statex {\bf end if}
			\Statex
			
		\hskip-1.5em \hrulefill
				\Statex\hskip-1em  {\bf R5: \emph{receiving a request} from $r$, resolving a conflict}
		\Statex upon reception of $(req_b, s\_num_r, id_r)$ do
		\Statex $req_p(b) \gets true;$
		\Statex $max\_rec_p \gets \max(max\_rec_p, s\_num_r)$
		\Statex {\bf if} $\Big(\big(\neg need_p(b)\big)$ {\bf or}
 $\big(p$ is $thirsty$ {\bf and} $(s\_num_r,id_r) < (s\_num_p, id_p) 
		\big)\Big)$ {\bf then}
		\Statexxx $[ hold_p(b)\gets false; Send(b)]$
		\Statex {\bf end if}
			\Statex
			
		\hskip-1.5em \hrulefill
		\Statex\hskip-1em  {\bf R6: \emph{receive bottle}}
		\Indent
		\Statexx upon reception of $b$ {\bf do}
		\Statexx $\quad$ $hold_p(b) \gets true$
		\EndIndent
	\end{algorithmic}
\end{algorithm}


\section{Multi-robot navigation as a drinking philosophers problem}\label{chp:map}
In this section we recast the multi-robot path execution problem as an instance of drinking philosophers problem. We first show that naive reformulation using existing DrPP solutions leads to conservative control policies. We then provide a solution that is based on Algorithm~\ref{alg:base}.

We first provide an intuitive explanation for the transformation from MRPE to DrPP, and then present the formal procedure.
Given a set $\Cells = \{\cell_1, \dots , \cell_{|\Cells|}\}$ of cells and a collection $\Pi = \{\traj{1}{}, \dots \traj{N}{}\}$ of paths, cells that appear in more than one path are called \emph{shared} and the rest are called \emph{free}. We denote the set of all shared cells by $\Cells_{shared}$. We assume that robots know all paths prior to the start of the execution, hence each robot knows which cells are shared and which ones are free. A shared cell must be occupied at most by one robot at any given time to avoid collisions. One can treat the robots as philosophers and shared cells as bottles to enforce this mutual exclusion requirement. In multi-robot setting, the actions \emph{``moving between two free cells"} or \emph{``occupying a free cell"} of a robot are mapped into the \emph{tranquil} state. Similarly, the \emph{``desire to move into a {shared cell}"} and \emph{``moving towards or occupying a shared cell} are mapped into the \emph{thirsty} and the \emph{drinking} states, respectively.

Given any two arbitrary robots, we define a bottle for each cell that is visited by both. For example, if the $k^{th}$ cell $v_k\in \Cells$ is visited both by $r_m$ and $r_n$, we define the bottle $\bottle{k}{m}{n}$. We denote the set of cells visited by both $r_m$ and $r_n$ by $\Cells_{m,n} \doteq \{v\mid  \exists\; t_m, t_n: \traj{m}{t_m}  = \traj{n}{t_n} = v \in \Cells_{shared}\}$. It must be noted that for a shared cell $v_k\in \Cells_{m,n}$, there exists a single bottle shared between $r_n$ and $r_m$, and both $\bottle{k}{m}{n}$ and $\bottle{k}{n}{m}$ refer to the same object. Multiple bottles would be defined for a shared cell that is visited by more than two robots, where each bottle is shared between exactly two robots. We use $\Bottles_{m,n}$ and $\Bottles_{m}$  to denote the set of all bottles $r_m$ shares with $r_n$ and with all other robots, respectively. We then define $\Bottles_m(V)$, the set of bottles associated with the cells in $ V\subseteq\Cells$ that $r_m$ share with others such that $\Bottles_m(V) \doteq \{\bottle{k}{m}{n} \in \Bottles_{m} \mid v_k \in V\}$. We use the following example to illustrate the concepts above. 

\begin{figure}
	\centering
	\includegraphics[width=0.8\linewidth]{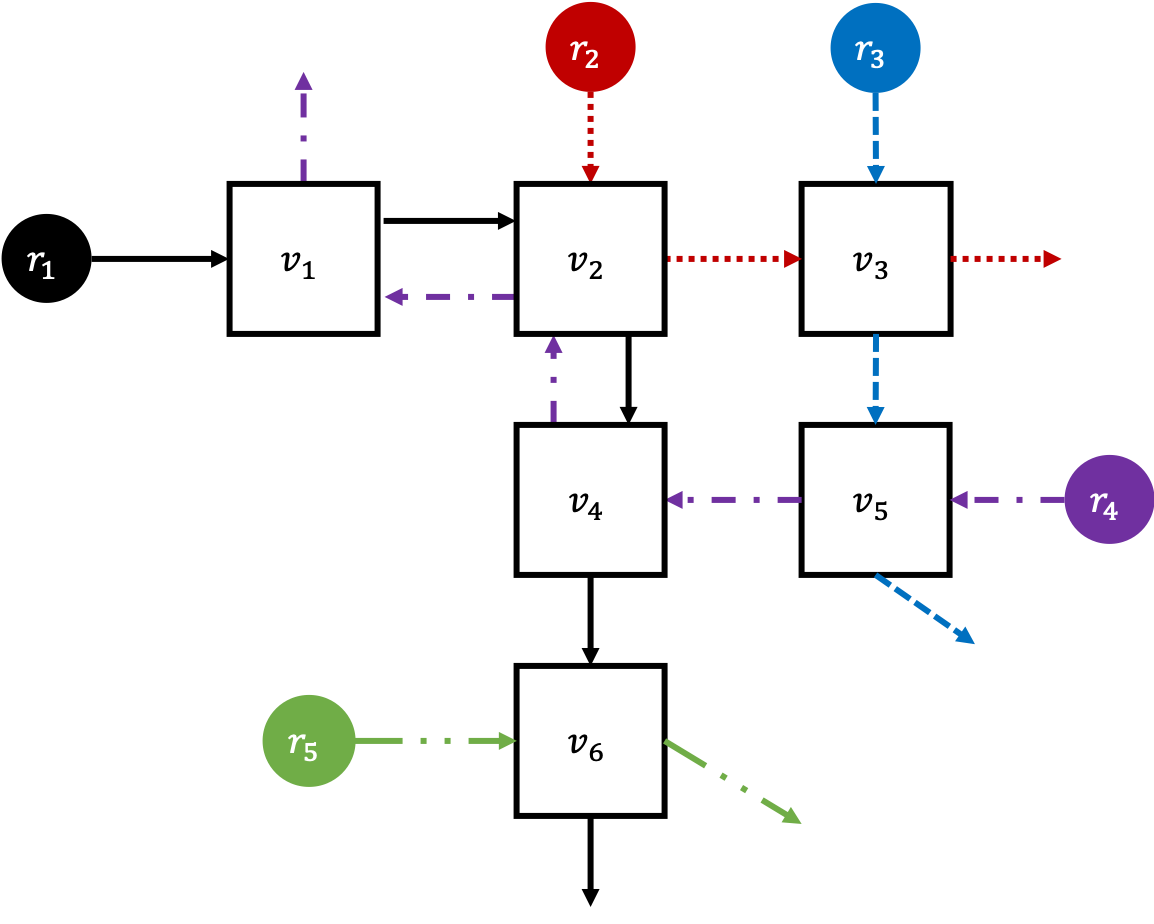}
	\caption{An illustrative example showing partial paths of five robots. Robots, each assigned a unique color/pattern pair, are initialized on free cells. Shared cells are shown as hollow black rectangles. Each path eventually reaches a free cell that is not shown for the sake of simplicity.}
	\label{fig:rainbow}
\end{figure}

\begin{example}\label{ex:1}
	In the scenario depicted in Figure~\ref{fig:rainbow}, the robot $r_1$ shares one bottle with $r_2$, $\Bottles_{1,2} = \{b^{2}_{1,2}\}$, three bottles with $r_4$,  $\Bottles_{1,4} = \{b^{1}_{1,4}, b^{2}_{1,4}, b^{4}_{1,4}\}$, and one bottle with $r_5$, $\Bottles_{1,5} = \{b^{6}_{1,5}\}$. The set $\Bottles_1$ is the union of these three sets, as $r_1$ does not share any bottles with $r_3$. Given $V= \{v_2\}$, we have $\Bottles_{1}(V) = \{b^{2}_{1,2}, b^{2}_{1,4}\}$.
\end{example}

Bottles are used to indicate the priority order between robots over shared cells. 
For instance, if $b^{k}_{m,n}$ is currently held by robot $r_m$, then \emph{$r_m$ has a higher priority than $r_n$ over the shared cell $v_k$}. Note that, this order is dynamic as bottles are sent back and forth. 

Collisions can be prevented simply by the following rule: \emph{``to occupy a shared cell $v_k$, the robot $r_n$ must be drinking from all the bottles in $B_n(v_k)$".} That is, $r_n$ must set all the bottles in $B_n(v_k)$ needed, and be in drinking state. Upon arriving a free cell, a drinking robot would become tranquil.  This rule prevents collisions as $r_n$ is the only robot allowed to occupy $v_k$ while it is drinking from $B_n(v_k)$. However, this rule is not sufficient to ensure that all robots reach their final cells. Without the introduction of further rules, robots might end up in a \emph{deadlock}. We formally define deadlocks as follows:

\begin{definition}\label{dfn:deadlock}
A \textbf{deadlock} is any configuration where a subset of robots, which have not reached their final cell, wait cyclically and choose $STOP$ action indefinitely.
\end{definition} 

To exemplify the insufficiency of the aforementioned rule, imagine the scenario shown in Figure~\ref{fig:rainbow}.  Robots $r_1$ and $r_4$ traverse the neighboring cells $v_1$ and $v_2$ in the opposite order. Assume $r_4$ is at $v_4$ and wants to proceed into $v_2$, and, at the same time, $r_1$ wants to move into $v_1$. Using the aforementioned rule, robots must be drinking from the associated bottles in order to move. Since they wish to drink from different bottles, both robots would be allowed to start drinking. 
After arriving at $v_1$, $r_1$ has to start drinking from $B_1(v_2)$ in order to progress any further. However, $r_4$ is currently drinking from $b_{1,4}^{2} \in B_1(v_2)$ and cannot stop drinking before leaving $v_2$. Similarly, $r_4$ cannot progress, as $r_1$ cannot release $b_{1,4}^{1}$ before leaving $v_1$. Consequently, robots would not be able to make any further progress, and would stay in drinking state forever.

\subsection{Naive Formulation}\label{ssec:naive}

We now show that deadlocks can be avoided by constructing the drinking sessions carefully. For the correctness of DrPP solutions, all drinking sessions must end in finite time. If drinking sessions are set such that \emph{a robot entering a shared cell is clear to move until it reaches a free cell without requiring additional bottles along the way}, then all drinking sessions would end in finite time. That is, if a robot is about to enter a segment which consists of consecutive shared cells, it is required to acquire not only the bottles associated with the first cell, but all the bottles on that segment. To formally state this requirement, let $\Session_n(t)$ denote the drinking session associated with cell $\pi_n^t$ for robot $r_n$. That is, $r_n$ should be drinking from all the bottles in $\Bottles_n(\Session_n(t))$ to occupy $\pi_n^t$. Now set

\begin{equation}\label{eq:naive}
\Session_n(t) = \pi_n^{t:t'} = \{\pi_n^t, \dots \pi_n^{t'}\}
\end{equation}

where $\pi_n^{k}\in \Cells_{shared}$ for all $k\in [t,t']$ and $\pi_n^{t'+1}\in \Cells_{free}$ is the first free cell after $\pi_n^{t}$. As long as it is drinking from $\Session_n(t)$, $r_n$ has the highest priority among all robots over the cells $\pi_n^{t:t'}$. Then, $r_n$ can constantly choose the action $GO$ until eventually reaching the free cell $\pi_n^{t'+1}$ and would stop drinking in finite time. Therefore, any existing DrPP solution, such as \cite{chandy-misra, ginat1989efficient, welch1993modular}, can be used to design the control policies that solve Problem~\ref{prob1} if the drinking sessions are constructed as in \eqref{eq:naive}. 

To illustrate, let us revisit the scenario in Figure~\ref{fig:rainbow}. To be able to occupy $v_1$, $r_1$ needs to acquire all bottles in $\Bottles_{1}(v_1, v_2, v_4, v_6)$. Then, $r_1$ is free to move all the way up to the last shared state $v_6$ without requiring additional bottles. Note however that, once $r_1$ reaches $v_2$, bottles $\Bottles_{1}(v_1)$ are no longer needed to avoid deadlocks. To allow more concurrency, we introduce a new rule that lets robots to \emph{drop} bottles they no longer need: 
\begin{itemize}[label = {}]
{\setlength\itemindent{-1em}	\item {\bf R7: upon leaving a shared state}}

		\item {\bf for} each $b \not\in\Bottles_{p}\Big( \Session_p\big(curr(p)\big)\Big)$ {\bf do} 
		{\setlength\itemindent{1em} \item  $need_p(b) \gets false$}
			{\setlength\itemindent{1em}\item {\bf if} $req_p(b)$ {\bf then}}
		 {\setlength\itemindent{2em} \item $[hold_p(b) \gets false; Send(b)]$}
		 	{\setlength\itemindent{1em}\item {\bf end if}}
		 \item {\bf end for}
\end{itemize}
Without {$R7$}, if $r_1$ is at $v_2$, 
$r_2$ would need to wait until $r_1$ leaves $v_6$. With {$R7$}, $r_1$ would release $b_{1,2}^2$ upon leaving $v_2$. Then, $r_2$ can move to $v_2$, while $r_1$ is at $v_4$.

However, the control policies resulting from the aforementioned approach are conservative and lead to poor performance in terms of both makespan and flowtime. To illustrate using the scenario shown in Figure~\ref{fig:rainbow}, if $r_1$ is currently at $v_1$, $r_5$ cannot move into $v_6$ since $b_{1,5}^{6}$ is held by $r_1$. This is a conservative action as $r_5$ cannot cause a deadlock by moving to $v_6$, as it moves to a free cell right after. 

We have seen that \emph{small} drinking sessions might lead to deadlocks, and \emph{large} drinking sessions might lead to unnecessary waits, and thus, bad performance. Our goal is to construct drinking sessions \emph{as small as possible} such that we can guarantee deadlock-freeness while allowing as much concurrent behavior as possible. To achieve this goal, we introduce a new drinking state and rules regarding its operation.

\subsection{New Drinking State and New Rules}\label{ssec:insat}

In this subsection, we propose 
a new drinking state for the philosophers, namely \emph{insatiable}. This new state is used when robot moves from a shared cell to another shared cell. We also add an additional rule $R8$ regarding this new state and modify the existing rules $R4$ and $R5$ of Algorithm~\ref{alg:base} as $R'4$ and $R'5$, respectively:

\begin{itemize}[label = {}]
	{\setlength\itemindent{-1em}	\item {\bf R'4: \emph{requesting a bottle}}}
\item {\bf if} $\big(need_p(b)$ {\bf and} $\neg hold_p(b)$ {\bf and} $req_p(b)\big)$ {\bf then}
		{\setlength\itemindent{1em}\item  $req_p(b)\gets false$}
	{\setlength\itemindent{1em}	\item $Send(req_b, s\_num_p, id_p, ds_p)$}
\item  {\bf end if}
\end{itemize}

\begin{itemize}[label = {}]
	{\setlength\itemindent{-1em}	\item {\bf {R'5}: \emph{receiving a request from $r$, and resolving a conflict} }}
	\item upon reception of $(req_b, s\_num_r, id_r, ds_r)$ do
	\item $req_p(b) \gets true;$
	\item $max\_rec_p \gets \max(max\_rec_p, s\_num_r)$
	\item  {\bf if} $\big((1)$ {\bf or} $(2)$ {\bf or} $(3)\big)$ 
		\item  {\bf then}
		{\setlength\itemindent{-1em} \item $\qquad$ $[hold_p(b)\gets false; Send(b)]$}
		\item  {\bf end if}
\end{itemize}
where 
\begin{enumerate}
	\item $\neg need_p(b)$
	\item $p$ is $thirsty$ {\bf and}
	\begin{enumerate}
		\item $\big(ds_r$ is $thirsty$ {\bf and} $(s\_num_r, id_r)\prec (s\_num_p, id_p)\big)$ {\bf or},
		\item $ds_r$ is $insatiable$,
	\end{enumerate} 
	\item  $p$ is $insatiable$ with $\big(\Bottles_p({\Session}_1), \Bottles_{p}({\Session}_2)\big)$  {\bf and},
	\begin{enumerate}
		\item $ds_r$ is $insatiable$ {\bf and},
		\item $b\not \in \Bottles_{p}(\mathcal{S}_1)$ {\bf and},
		\item $(s\_num_r, id_r)\prec (s\_num_p, id_p)$.
	\end{enumerate}
\end{enumerate}

\begin{itemize}[label = {}]
	{\setlength\itemindent{-1em}	\item {\bf R8: becoming insatiable with tuple $\big(\Bottles_p({\Session}_1), \Bottles_{p}({\Session}_2)\big)$}} 
		\item {\bf for}  each bottle $b\in\Bottles_{p}({\Session}_2)$ {\bf do}  
			{\setlength\itemindent{1em}\item $need_p(b) \gets true$}
		\item {\bf end for}  
		\item  {\bf if} not holding all bottles in $\Bottles_{p}(\Session_1\cup \Session_2)$ {\bf then}
			{\setlength\itemindent{1em}\item become {$insatiable$}}
		\item  {\bf end if}
\end{itemize}
The message structure used for requesting bottles is modified in $R'4$ and now includes the drinking state $ds_p \in \{tranquil$, $thirsty$, $drinking$, $insatiable\}$ of the sender. This information is used in conjunction with $id$ and $s\_num$ by the receiver to decide whether to grant or defer the request.

In the naive formulation, drinking sessions are set such that a robot entering a shared cell is free to move until it reaches a free cell, without requiring additional bottles along the way. The insatiable state is intended to soften this constraint. Assume robot $r_n$ wants to move to shared cell $\pi_n^t$, and the first free cell after $\pi_n^t$ is $\pi_n^{t'+1}$ for some arbitrary $t'>t$, all the cells in between are shared. If $r_n$ enters the first shared cell without acquiring all the bottles until $\pi_n^{t'+1}$, it would need to acquire those bottles at some point along the way. If $r_n$ becomes thirsty to acquire those bottles, it risks losing the bottles it currently holds. If another robot $r_m$ with a higher priority needs and receives the bottles associated with the cell $r_n$ currently occupies, two robots might collide.

Insatiable state allows a robot to request new bottles without risking to lose any of the bottles it currently holds.
In this state, the robot does not hold all the bottles it needs to start drinking, similar to thirsty state. The difference between two states is that, an insatiable philosopher always has a higher priority than a thirsty philosopher regardless of their session numbers. Moreover, an insatiable philosopher does not release under any circumstance any of the bottles needed to occupy the cell it is currently in. 

In other words, if $p$ is insatiable with $(\Bottles_{1}, \Bottles_{2})$, then $p$ already has all bottles in $\Bottles_{1}$ and cannot release them. Moreover, $p$ is trying to acquire the bottles in $\Bottles_{2}$ to start drinking, which might be released, if a robot with higher priority requests them. An insatiable robot always has a higher priority than a thirsty robot. In case of identical drinking states, $\prec$ relation is used to resolve the priority order.

The insatiable state and the rules regarding its operation might lead to deadlocks without careful construction of drinking sessions. We now explain how to construct drinking sessions to avoid deadlocks.

\subsection{Constructing Drinking Sessions}\label{ssec:sessions}

To compute drinking sessions, we first need to define a new concept called \emph{\graphname}:
\begin{definition}
	The \textbf{\graphname} induced by the collection $\Pi = \{\traj{1}{}, \dots \traj{N}{}\}$ of paths is a directed edge-colored multigraph $G_\Pi = (\Cells, \Edges_{\Pi}, \mathcal{C})$ where $\Cells$ is a set of nodes, one per each cell in $\Pi$, $\Edges_{\Pi} = \{( \traj{n}{t}, c_n,  \traj{n}{t+1}) \mid \pi_n\in \Pi\}$ is the set of edges, representing transitions of each path, and $\mathcal{C} = \{c_1,\dots, c_N\}$ is the set of colors, one per each path (i.e., one per each robot).	
\end{definition}

A \graphname~is a graphical representation of a collection of paths, overlayed on top of each other. The nodes of this graph correspond to discrete cells that partition the workspace, and edges illustrate the transitions between them. Color coding of edges indicate which robot is responsible for a particular transition. In other words, if $\pi_n$ has a transition from $u$ to $v$, then there exists a $c_n$ colored edge from $u$ to $v$ in $G_\Pi$, i.e.,  $(u,c_n,v)\in \Edges_{\Pi}$.

\graphname s are useful to detect possible deadlock configurations. Intuitively, deadlocks occur when a subset of robots wait cyclically for each other. We first show that such configurations correspond to a \emph{rainbow cycle} in the corresponding \graphname. A rainbow cycle is a closed walk where no color is repeated. Let $\Pi$ be a collection of paths and $G_{\Pi}$ be the \graphname~induced by it. Assume that a subset $\{r_1, \dots, r_K\}\subseteq \mathcal{R}$ of robots are in a deadlock configuration such that $r_n$ waits for $r_{n+1}$ for all $n \in \{1, \dots, K\}$ where $r_{K+1} = r_1$. That is, $r_n$ cannot move any further, because it wants to move to the cell that is currently occupied by $r_{n+1}$. Let $v_n$ denote the current cell of $r_n$. Since $r_n$ wants to move from $v_n$ to $v_{n+1}$, we have $e_n = (v_n, c_n, v_{n+1}) \in \Edges_{\Pi}$. Then, $\omega = \{(v_1, c_1, v_2),\dots ,(v_K, c_K, v_1)\}$ is a rainbow cycle of $G_\Pi$.  For instance, there are two rainbow cycles in Figure~\ref{fig:rainbow}: $\omega_1 = \{(v_1, c_1, v_2), (v_2, c_4, v_1)\}$ and $\omega_2 = \{(v_2, c_1, v_4), (v_4, c_4, v_2)\}$.

The first idea that follows from this observation is to limit the number of robots in each rainbow cycle to avoid deadlocks. However, this is not enough as rainbow cycles can intersect with each other and robots might end up waiting for each other to avoid eventual deadlocks. For instance, in the scenario illustrated in Figure~\ref{fig:rainbow}, let $r_1$ and $r_4$ be at $v_1$ and $v_4$, respectively. The number of robots in each rainbow cycles is limited to one, nonetheless, this configuration will eventually lead to a deadlock.

 \begin{algorithm}[t]
	\caption{find\_equivalence\_classes}\label{alg:cycles}
	\hspace*{\algorithmicindent} \textbf{Input} $G_{\Pi}$
	\textbf{ return} $\tilde{G}_{\Pi}$
	\begin{algorithmic}[1]
		\State $\sim\gets\emptyset$
		\For{$u \in G_{\Pi}$} 
		\State expand $\sim$ such that $(u, u)$ $\in$ $\sim$
		\EndFor
		\State {$\mathcal{W} \gets $ find\_rainbow\_cycles($G_{\Pi}$)}
		\If{$\mathcal{W} = \emptyset$}
		\State $\tilde{G}_{\Pi}\gets G_{\Pi}$
		\State\Return 
		\Else
		\For {$W\in \mathcal{W}$}
		\For{$u,v \in W$} 
		\State expand $\sim$ such that $(u, v)$ $\in$ $\sim$
		\EndFor
		\EndFor
		\State $\tilde{G}_{\Pi}\gets find\_quotient(G_{\Pi}, \sim)$ 
		\State $find\_equivalence\_classes(\tilde{G}_{\Pi})$
		\EndIf
	\end{algorithmic}
\end{algorithm}

We propose Algorithm~\ref{alg:cycles} to construct the drinking sessions, which are used to prevent such deadlocks. Given a collection $\Pi$ of paths let $G_{\Pi} = (\Cells, \Edges_{\Pi}, \mathcal{C})$ denote its \graphname. We first define equivalence relation $\sim$ on $\Cells$ such that each node is equivalent only to itself. We then find all rainbow cycles in $G_{\Pi}$. Let $\mathcal{W}$ denote the set of all rainbow cycles. For each rainbow cycle $W\in\mathcal{W}$, we expand the equivalence relation $\sim$ by declaring all nodes in $W$ to be equivalent. That is, if $u$ and $v$ are two nodes of the rainbow cycle $W$, we add the pair $ (u,v)$ to the equivalence relation $\sim$. Note that, due to transitivity of the equivalence relation, nodes of two intersecting rainbow cycles would belong to the same equivalence class. The relation $\sim$ partitions $\Cells$ by grouping the intersecting rainbow cycles together. We then find the quotient set $\Cells\big/{\sim}$ and define a new graph $\tilde{G}_{\Pi} = (\Cells\big/{\sim}, \tilde{\Edges}_{\Pi}, \mathcal{C})$ where $([u], c_m, [v])\in \tilde{\Edges}_{\Pi}$ if $[u]\not=[v]$, and there exists $\alpha \in [u]$, $\beta \in [v]$ such that $(\alpha,c_m, \beta)\in \Edges_{\Pi}$. That is, we create a node for each equivalence class. We then add a $c_m$ colored edge to $\tilde{G}_{\Pi}$ between the nodes corresponding $[u]$ and $[v]$ if there is a $c_m$ colored edge in $G_{\Pi}$ from a node in $[u]$ to a node in $[v]$. We repeat the same process with $\tilde{G}_{\Pi}$ in a recursive manner until no more rainbow cycles are found.

\begin{prop}
	Algorithm \ref{alg:cycles} terminates in finite steps.
\end{prop}
\begin{proof}
Since all paths are finite, the number of nodes in the \graphname~$G_{\Pi}$, $|\Cells|$, is finite. At each iteration, Algorithm \ref{alg:cycles} either finds a new graph $\tilde{G}_{\Pi}$ which has a smaller number of nodes, or returns $G_{\Pi}$. Therefore, Algorithm \ref{alg:cycles} is guaranteed to terminate at most in $|\Cells|$ steps. 
\end{proof}

\begin{remark}
Algorithm \ref{alg:cycles} needs to find all rainbow cycles of an edge-colored multi-graph at each iteration, which can be done in the following way. Given $G = (\Cells, \Edges, \mathcal{C})$, obtain $E \subseteq \Cells\times \Cells$ from $\Edges$ by removing the coloring and replacing multiple edges between the same two nodes with a single edge. Then, find all simple cycles in the graph $G' = (\Cells, E)$. Finally, check if these cycles can be colored as a rainbow cycle.

Finding all simple cycles in a directed graph is time bounded by $\mathcal{O}((|V| + |E|)(C+1))$ and space bounded by $\mathcal{O}((|V| + |E|)$, where $C$ is the number of cycles \cite{johnson1975finding}. Although the number of cycles in a directed graph grows, in the worst-case, exponentially with the number of nodes, this operation can be done efficiently in practice \cite{gupta2021finding}. Deciding if a cycle can be rainbow colored can be posed as an exact set cover problem, which is NP-complete. This is essentially due to the fact that, in the worst-case, the number of cycles in a multi-graph can be exponential in the number of colors compared to the corresponding directed graph. However, the number of nodes decrease at each iteration of Algorithm \ref{alg:cycles}, making computations easier. Moreover, while the worst-case complexity is high, these operations can usually be performed efficiently in practice.
\end{remark}

When the Algorithm~\ref{alg:cycles} finds the fixed point, we set
\begin{equation}\label{eq:session}
\tilde{\Session}_n(t) \doteq \Session_{n}(t)\cap [\pi_n^{t} ]
\end{equation}
where $\Session_{n}(t)$ is defined as in \eqref{eq:naive} and $[\pi_n^{t}]$ is the equivalence class of $\pi_n^{t}$. That is, $r_n$ must be drinking from all the bottles in $\Bottles_n(\tilde{\Session}_n(t))$ to be able to occupy $\pi_n^t$. 

We now revisit the example in Figure~\ref{fig:rainbow}.

\begin{example}
	Let $G_{\Pi}$ be given as in Figure~\ref{fig:rainbow}. After the first recursion of Algorithm~\ref{alg:cycles}, $[v_1] = \{ v_1, v_2, v_4\}$ and $[v_i] = \{v_i\}$ for $i\in\{3,5,6\}$. After the second recursion, $[v_1] = \{v_1, v_2, v_3, v_4, v_5\}$ and $[v_6] = \{v_6\}$. No rainbow cycles are found after the second recursion, therefore, $\tilde{\Session}_1(1) = \{v_1, v_2, v_4, v_6\}\cap \{v_1, v_2, v_3, v_4, v_5\} = \{ v_1, v_2, v_4\}$.
\end{example}

While all cells except $v_6$ get merged into a single cell in Figure~\ref{fig:rainbow}, constructing drinking sessions as in \eqref{eq:session} allows multiple robots to simultaneously occupy $v_1-v_5$. For instance, the drinking sessions of $r_1$ and $r_3$ are disconnected since they have no common cells. Therefore, $r_1$ and $r_3$ can enter cells $v_1$ and $v_3$, respectively, at the same time. Similarly, such drinking sessions also allow $r_1$, $r_2$ and $r_3$ to simultaneously occupy cells $v_4$, $v_2$ and $v_5$, respectively. Even in this small example, we can see the benefit of using Eq.~\eqref{eq:session} instead of Eq.~\eqref{eq:naive}. The modification allows $r_1$ and $r_5$ to be at $v_4$ and $v_6$, respectively, while it was not possible with the naive formulation.

\begin{remark}\label{rem:conservative}
	Sessions constructed by \eqref{eq:session} are always contained in the sessions constructed by \eqref{eq:naive}. That is, when drinking sessions are found as in \eqref{eq:session}, robots would need fewer bottles to move, and the resulting control policies would be more permissive. 
\end{remark}

We now propose a control policy that prevents collisions and deadlocks when drinking sessions are constructed as in \eqref{eq:session}.

\subsection{Control Strategy}\label{ssec:control}

We propose Algorithm~\ref{alg:main} as a control policy to solve Problem~\ref{prob1}. In order to implement Algorithm~\ref{alg:main} in a distributed manner, we require the communication graph to be identical to the resource dependency graph. That is, if two robots visit a common cell, there must be a communication channel between them. We also assume that messages from one robot to another are received in the order that they are sent. We now briefly explain the flow of the control policy, which is illustrated in Figure~\ref{fig:flow}, and then provide more details. 

Robots are initialized as follows. If robot $r_n$ starts at a shared cell, all bottles in its initial drinking session are given to $r_n$ and the related request tokens are given to the corresponding robots. To ensure bottles can be assigned in this way, we require that the initial drinking sessions are disjoint. Robots with shared initial cells are then initialized in drinking state, which is possible since they hold all the required bottles, and the remaining robots are initialized in tranquil state.

If the final cell is reached, $STOP$ action is chosen as the robot accomplished its task. Otherwise, if a robot is in either tranquil or drinking state, the control policy chooses the action $GO$ until the robot reaches to the next cell. When a robot moves from a free cell to a shared cell, it first becomes thirsty and the control policy issues the action $STOP$ until the robot starts drinking. When moving between shared cells, a robot becomes insatiable if it needs to acquire additional bottles, and $STOP$ action is chosen until the robot starts drinking again. If a robot $r_n$ is leaving a shared cell where another robot $r_m$'s path terminates, for the last time, $r_n$ sends $r_m$ a {\emph{cleared}} message. When a robot's path terminates at a shared cell, it must be careful not to arrive early and block others from progressing. Therefore, when a robot is about to move to a segment of consecutive shared cells which includes its final cell, it needs to wait for others to clear its final cell. 

All robots are initialized as previously described. Let $r_n$ be an arbitrary robot. Lines $1-2$ of Algorithm~\ref{alg:main} ensure that $r_n$ does not move after reaching its final cell. Otherwise, let $\pi_n^t$ denote the next cell on $r_n$'s path. If $\pi_n^t$ is a free cell, the control policy chooses the $GO$ action until the robot reaches $\pi_n^{t+1}$ (lines $5-8$). $r_n$ goes back to tranquil state if it was drinking and sends a \emph{cleared} message to a corresponding robot $r_m$ if (i) $\pi_n^{t-1}$ was the terminal cell for $r_m$ and (ii) $\pi_n^{t-1}$ will not be visited by $r_n$ again in the future (lines $9-13$). When $\pi_n^t$ is a shared cell, there are two possible options: (i) If there is no free cell between the next cell and the final cell of $r_n$, i.e., $\pi_n^{end} \in \Session_n(t)$ where $\Session_n(t)$ is defined as in \eqref{eq:naive}, the robot must wait for all other robots to clear this cell (lines $15-18$). This wait is needed, otherwise, $r_n$ might block others by arriving and staying indefinitely at its final cell. When all others clear its final state, $r_n$ can start moving again. (ii) If the final cell is not included in the drinking session, $r_n$ checks its drinking state. If tranquil, $r_n$ becomes thirsty with the drinking session $\tilde{\Session}_n(t)$ (lines $19-20$). If drinking, $r_n$ becomes insatiable with $\Big(\Bottles_n\big(\tilde{\Session}_n(t)\big), \Bottles_n\big(\tilde{\Session}_n(t+1)\big)\Big)$ (lines $21-22$). Then, the robot waits until it starts drinking to move to the next cell (lines $24-26$). When the robot starts drinking, it is allowed to move until it reaches $\pi_n^t$ (lines $27-29$). Upon reaching $\pi_n^t$, the robot sends cleared signal if needed (line $31$), as previously explained.

We now show the correctness of Algorithm~\ref{alg:main}. 

\begin{theorem}\label{theorem1}
	Given an instance of Problem \ref{prob1}, using Algorithm \ref{alg:main} as a control policy solves Problem \ref{prob1} if 
	\begin{enumerate}
		\item Initial drinking sessions are disjoint for each robot, i.e., $\tilde{\Session}_m(0) \cap \tilde{\Session}_n(0) = \emptyset$ for all $m\neq n$ and
		\item Final cells of each robot belong to a different equivalence class, i.e., $(\pi_m^{end},\pi_n^{end}) \not\in \sim$ for any $m\neq n$ where $\sim$ is computed according to Algorithm~2, and
		\item There exists at least one free cell in each $\pi_n$.
	\end{enumerate}
\end{theorem}
The proof of Theorem~\ref{theorem1} can be found in the Appendix.

The first condition in Theorem~\ref{theorem1} ensures that robots can be initialized correctly. Imagine a robot $r_n$ whose path starts with a shared cell. If $r_n$ is initialized in tranquil state, it will momentarily violate the requirement that \emph{``$r_n$ must be drinking from all the bottles in $\Bottles_n(\tilde{\Session}_n(0))$ to be able to occupy $\pi_n^0$"}. If initial drinking sessions are disjoint for each robot, $r_n$ can immediately start drinking. Therefore, all such robots can be initialized in drinking state if the first condition is satisfied. The second condition is required so that robots whose paths end in a shared cell do not block others from progressing by reaching their final cells early. The last condition is required otherwise (1) and (2) cannot be satisfied at the same time.

\begin{remark}
If (1) of Theorem~\ref{theorem1} is replaced by $${\Session}_m(0) \cap {\Session}_n(0)  = \emptyset,$$ the Naive formulation explained in Section~\ref{ssec:naive} also solves Problem \ref{prob1}. However, as mentioned in Remark~\ref{rem:conservative}, drinking sessions constructed by \eqref{eq:naive} always contain sessions constructed by \eqref{eq:session}, that is, $\Session_{n}(t)\supseteq \tilde{\Session}_{n}(t)$. Therefore, condition (1) of Theorem~\ref{theorem1} becomes more restrictive for the Naive implementation.
\end{remark}

\begin{figure}[t]
	\centering
	\includegraphics[width=1\columnwidth]{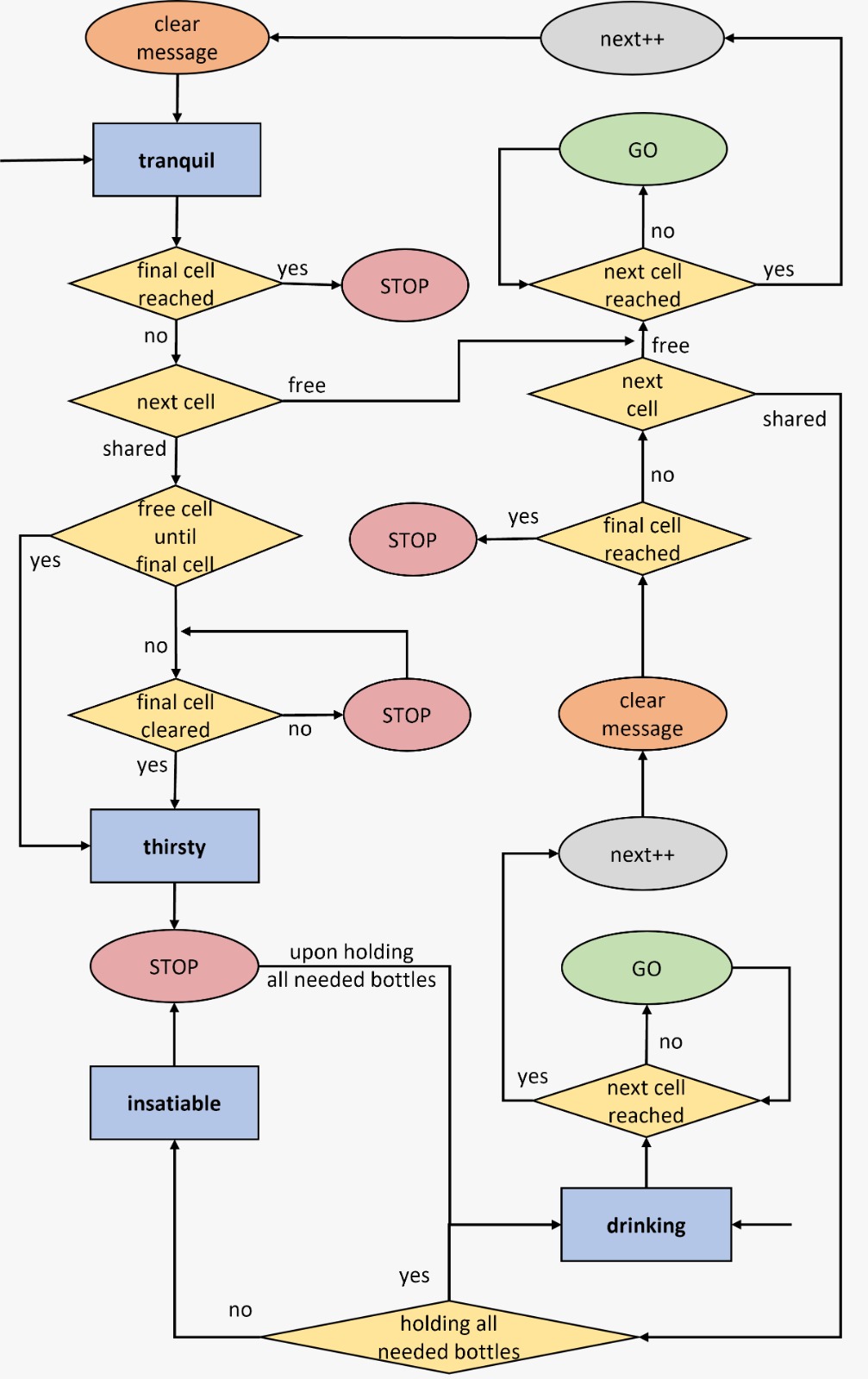}
	\caption{Flowchart of the control policy explained in Algorithm \ref{alg:main}.}
	\label{fig:flow}
\end{figure}

 \begin{algorithm}
	\caption{Control policy for $r_n$}\label{alg:main}
	\begin{algorithmic}[1]
		\If{$r_n$.is\_final\_cell\_reached}
		\State $r_n.STOP$
		\Else
		\State $t$ $\gets$ $next(r_n)$
		\If{is\_free($\pi_n^t$)}
		\While{$\neg r_n$.is\_reached($\pi_n^t$)}
		\State $r_n.GO$
		\EndWhile
		\State $next(r_n)$ $\gets$ $next(r_n)+1$
		\If{$r_n$.is\_drinking}
		\State $r_n$.get\_tranquil() 
		\State send\_cleared\_message\_if\_needed()
		\EndIf
		\Else
		\If{$\pi_n^{end} \in \Session_n(t)$}
		\While{ $\neg$cleared($\pi_n^{end}$)}
		\State $r_n.STOP$
		\EndWhile
		\ElsIf{$r_n$.is\_tranquil}
		\State $r_n$.get\_thirsty($\tilde{\Session}_n(t)$)
		\ElsIf{$r_n$.is\_drinking}
		\State $r_n$.get\_insatiable($\Bottles_n(	\tilde{\Session}_n(t), \tilde{\Session}_n(t+1))$)
		\EndIf
		\While{$\neg r_n$.is\_drinking}
		\State $r_n.STOP$
		\EndWhile
		\While{$\neg r_n$.is\_reached($\pi_n^t$)}
		\State $r_n.GO$
		\EndWhile
		\State $next(r_n)$ $\gets$ $next(r_n)+1$
		\State send\_cleared\_message\_if\_needed()
		\EndIf
		\EndIf
	\end{algorithmic}
\end{algorithm}

\begin{remark}
The control policy given in Algorithm~\ref{alg:main} satisfies liveness, fairness and concurrency properties. The liveness proof is shown in Theorem~\ref{theorem1}, and fairness and concurrency properties follows directly from \cite{ginat1989efficient}.
\end{remark}


\section{Results}\label{chp:results}
In this section, we compare our \emph{Rainbow Cycle} based method explained in Sections~\ref{ssec:insat}-\ref{ssec:control}, denoted \emph{DrPP-RC}, with other path execution methods using identical paths. Firstly, we compare DrPP-RC with the \emph{Naive} method, denoted \emph{DrPP-N}, which is explained in Section~\ref{ssec:naive}. This comparison demonstrates the performance improvement that results from the addition of the new drinking state. As stated in Remark~\ref{rem:conservative}, DrPP-RC uses smaller drinking sessions, and allows more concurrency. We also provide results {for DrPP-N \emph{without} $R7$}. 
This rule is an addition to the original DrPP solution of \cite{chandy-misra} and exploits the structure of the multi-robot path execution problem by allowing robots to drop bottles while in drinking state. 

We further compare DrPP-RC with the \emph{Minimal Communication Policy} of \cite{ma2017multi}, denoted \emph{MCP}, which prevents collisions and deadlocks by maintaining a fixed visiting order for each cell. A robot is allowed to enter a cell only if all the other robots, which are planned to visit the said cell earlier, have already visited and left the said state. It is shown that, under mild conditions on the collection of the paths, keeping this fixed order prevents collisions and deadlocks. We refer the reader to \cite{ma2017multi} for more details. 

We also note that, conditions required by \cite{zhou2020distributed} are too restrictive for the majority of the examples provided in this section. That is, some nodes in \graphname~are connected by more than one colored edge, hence \cite{zhou2020distributed} cannot be used. On the other hand, merging shared cells as in Equation~\eqref{eq:naive} generates a quotient graph that satisfies the required conditions. Then, the performance of \cite{zhou2020distributed} is identical to that of DrPP-N without $R7$. However, as Section~\ref{ssec:random} shows,  \cite{zhou2020distributed} cannot still solve all problems solved by DrPP-RC, and for the problems it can solve, it is significantly outperformed by both DrPP-RC and DrPP-N.

To capture the uncertainty in the robot motions, each robot is assigned a \emph{delay probability}. When the action $GO$ is chosen, a robot either stays in its current cell with this probability, or completes its transition to the next cell before the next time step leading to asynchrony between robots' motion. Our implementation can be accessed from \url{https://github.com/sahiny/philosophers}.

\subsection{Randomly Generated Examples}\label{ssec:random}
\begin{figure}
	\centering 
	\includegraphics[width=1\linewidth, 
	clip]
	{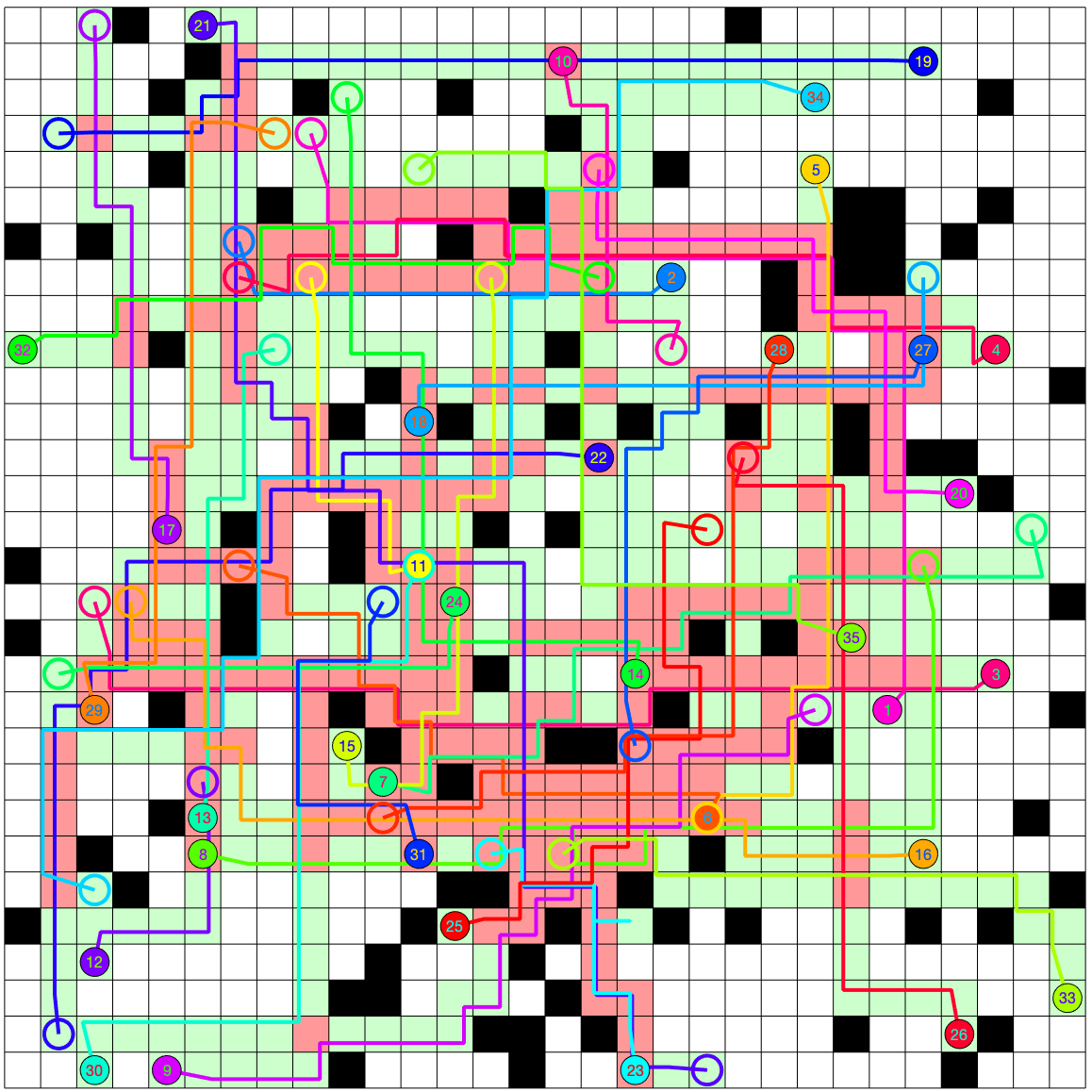}
	\caption{Randomly generated example, denoted random1, consisting of $35$ robots on a $30\times30$ grid with $10\%$ blocked cells, which are shown in black. The path of each robot is shown with a unique color where the solid and hollow circles represent the initial and final cells, respectively. Free (i.e., used by a single robot) and shared (i.e., visited by more than one robot) cells are painted green and red, respectively. Among the cells that are visited by at least one robot, 44\% (217/496) are shared cells, and a robot's path consists of 65\% shared cells on average. These statistics are similar for other random examples.}\label{fig:random1} 
\end{figure}

There are $10$ MRPE instances in \cite{ma2017multi}, labelled random~1-10, where $35$ robots navigate in 4-connected grids of size $30\times30$. In each example, randomly generated obstacles block $10\%$ of the cells, and robots are assigned random but unique initial and final locations. The first of these randomly generated examples can be seen Figure~\ref{fig:random1}. All control policies use the same paths generated by the Approximate Minimization in Expectation algorithm of \cite{ma2017multi}. 
Delay probabilities of robots are sampled from the range $(0, 1-1/t_{max})$. Note that, higher delay probabilities can be sampled as $t_{max}$ increase, resulting in \emph{slow moving} robots. Figure~\ref{fig:random_results} reports the makespan and flowtime statistics averaged over 1000 runs for varying $t_{max}$ values. The delay probabilities are sampled randomly for each run, but kept identical over different control policies. As expected, both makespan and flowtime statistics increase with $t_{max}$, as higher delay probabilities result in slower robots. An illustrative run of the DrPP-RC algorithm for $t_{max} = 2$ and environment random1 can be seen from \url{https://youtu.be/tht4ydW5iJA}.

From Figure~\ref{fig:random_results}, we first observe that the addition of $R7$ improves the flowtime performance of DrPP-N significantly, while its effect on makespan is neglible. Secondly, we observe that DrPP-RC always performs better than DrPP-N. This is expected as drinking sessions for DrPP-N, which are computed by \eqref{eq:naive}, are always larger than the ones of DrPP-RC, which are computed by \eqref{eq:session}. That is, robots using DrPP-N need more bottles to move, and thus, wait more. Moreover, DrPP-N requires stronger assumptions to hold for a collection of paths. For instance, only one of the ten random examples (random7) satisfy the the assumptions in Theorem~\ref{theorem1} for DrPP-N. The number of instances that satisfy the assumptions increase to four for DrPP-RC (random 3, 4, 7, 10). The random1 example illustrated in Figure~\ref{fig:random1} originally violates the assumptions, but this is fixed for both drinking based methods by adding a single cell into a robot's path. We here note that, the set of valid paths for MCP and DrPP algorithms are non-comparable. There are paths that satisfy the assumptions of one algorithm and violate the other, and vice versa.

We also observe that makespan values are quite similar for DrPP-RC and MCP methods, although MCP often performs slightly better in this regard. Given a collection of paths, the makespan is largely determined by the \emph{``slowest''} robot, a robot with a long path and/or a high delay probability, regardless of the control policies. Therefore, makespan statistics do not necessarily reflect the amount of concurrency allowed by the control policies. Ideally, in the case of a slow moving robot, we want the control policies not to stop or slow down other robots unnecessarily, but to allow them move freely. The flowtime statistics reflect these properties better. From Figure~\ref{fig:random_results}, we see that flowtime values increase more significantly with $t_{max}$ for MCP, compared to DrPP-RC. This trend can be explained with how priority orders are maintained in each of the algorithms. As the delay probabilities increase, there is more uncertainty in the motion of robots. MCP keeps a fixed priority order between robots, which might lead to robots waiting for each other unnecessarily. On the other hand, DrPP-RC dynamically adjusts this order, which leads to more concurrent behavior, hence the smaller flowtime values. Section~\ref{ssec:nasty} illustrates this phenomenon with a simple example.

\begin{figure*}[h]
	\centering
	\includegraphics[width=1\textwidth, clip]{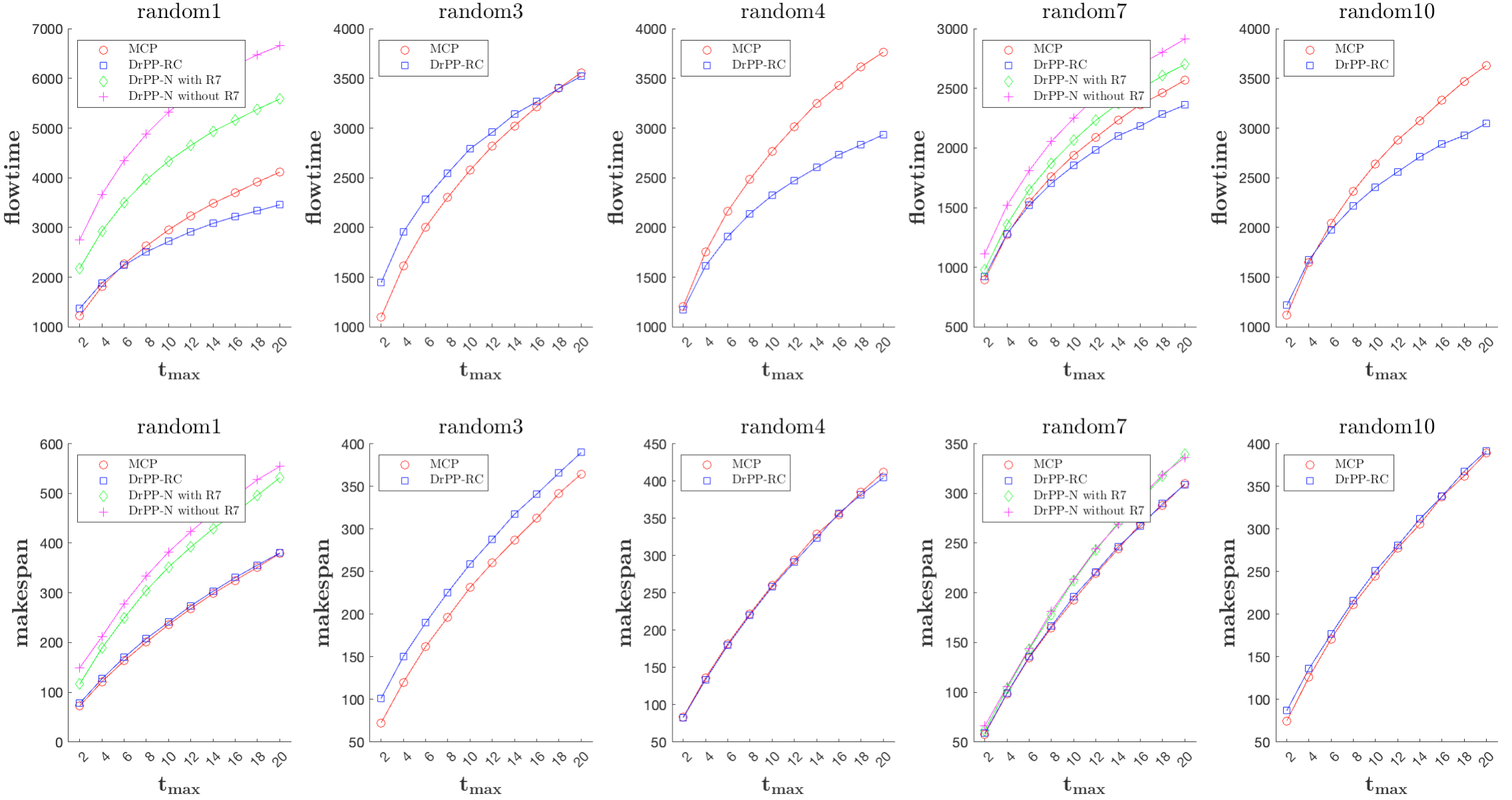}
	\caption{Makespan and flowtime statistics averaged over 1000 runs for the randomly generated environments under varying $t_{max}$ values. DrPP based method cannot be used in environments where the collection of paths violate the conditions in Theorem~\ref{theorem1}. DrPP-N and DrPP-RC methods can solve 2 and 5 out of 10 randomly generated instances, respectively.}
	\label{fig:random_results}
\end{figure*}

\begin{figure} 
	\centering 
	\includegraphics[width=0.7\columnwidth, trim = {5cm 2cm 5cm 4cm}, clip]{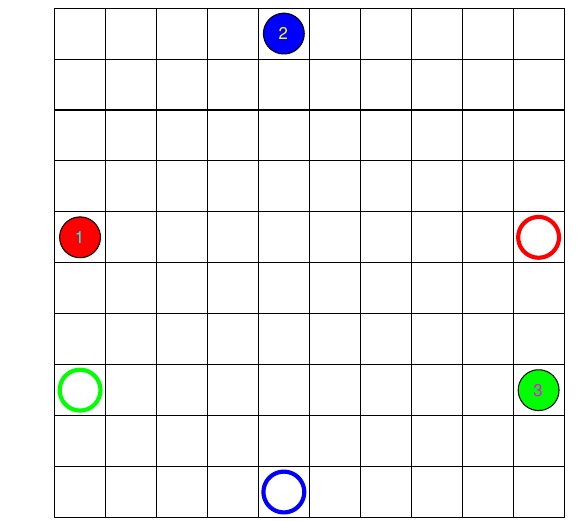}
	\caption{A simple example to show effects of a slow moving robot. Robots $r_1$, $r_2$ and $r_3$ are colored in red, blue and green, respectively. Initial and final cells of the robots are marked with solid and hollow circles of their unique color, respectively.}\label{fig:nasty} 
\end{figure}

\subsection{Makespan versus Flowtime}\label{ssec:nasty}

As mentioned earlier, \cite{ma2017multi} assumes that delay probabilities are known a priori, and computes paths to minimize the expected makespan. Once the paths are computed, the priority order between robots is fixed to ensure MCP policies are collision and deadlock-free. We now provide a simple example to illustrate the effect of using inaccurate delay probabilities in the path planning process. Imagine $3$ robots are sharing a $10$ by $10$ grid environment as shown in Figure~\ref{fig:nasty}. Assume that the delay probabilites for robots $r_1$, $r_2$ and $r_3$ are known to be $\{0, 0.4, 0.8\}$, respectively. If we compute paths to minimize the expected makespan, resulting paths are straight lines for each robot. Paths $\pi_1$ and $\pi_2$ intersect at a single cell, for which $r_1$ has a priority over $r_2$. Similarly $\pi_2$ and $\pi_3$ also intersect at a single cell, for which $r_2$ has a priority over $r_3$. We run this example using inaccurate delay probabilities $\{0.8, 0.4, 0\}$ to see how the makespan and flowtime statistics are affected. 

Over $1000$ runs, makespan values are found to be $48.30$ and $45.77$ steps for MCP and DrPP-RC implementations, respectively. The makespan values are close because of the slow moving $r_1$, which becomes the bottleneck of the system. Therefore, it is not possible to improve the makespan statistics by employing different control policies. However, the flowtime statistics are found as $128.78$ and $77.78$ steps for MCP and DrPP-RC implementations, respectively. Significant difference is the result of how a slow moving robot is treated by each policy. For the MCP implementation, $r_2$ (resp. $r_3$) needs to wait for $r_1$ (resp. $r_2$) unnecessarily, since the priority order is fixed at the path planning phase. On the other hand, DrPP-RC implementation allows robots to modify the priority order at run-time, resulting in improved flowtime statistics.

\begin{figure}
	\centering 
	\includegraphics[width=1\columnwidth, trim = {6.3cm 9.1cm 6.3cm 8.3cm}, clip]{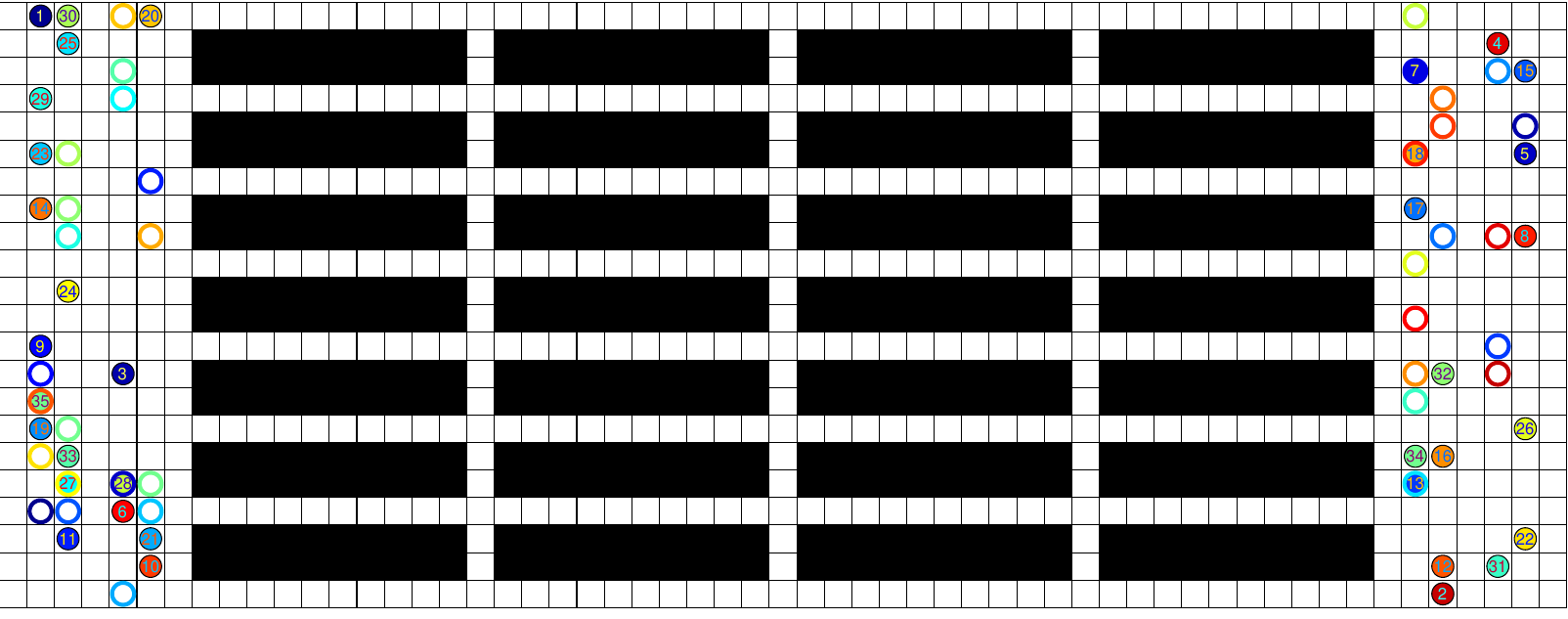}
	\caption{Illustration of a warehouse example on a $22\times 57$ grid. Blocked cells are shown in black. Initial and final cells are marked with a solid and a hollow circle of a unique color, respectively.}\label{fig:warehouse} 
\end{figure}

\subsection{Warehouse Example}\label{ssec:warehouse}

We also compare the performance of the control policies in a more structured warehouse-like environment. This warehouse example is taken from \cite{ma2017multi}, and it has $35$ robots as shown in Figure~\ref{fig:warehouse}. The makespan and flowtime statistics are reported in Figure~\ref{fig:warehouse_results}, which are averaged over 1000 runs for varying $t_{max}$ values. Due to stronger assumptions on the collection of paths, DrPP-N is not able to handle this example. Similarly, the conditions required by \cite{zhou2020distributed} are too restrictive, hence it cannot solve this problem. Although paths can be altered to allow \cite{zhou2020distributed} to be used, this requires each aisle to be abstracted as one discrete cell, and limits the number of robots in each aisle to at most one. As a result, the performance of \cite{zhou2020distributed} would be significantly worse compared to both DrPP-RC and MCP, no matter how paths are generated.

Similar to Section~\ref{ssec:random}, we observe that makespan values are better for MCP, but DrPP-RC scales better with $t_{max}$ for flowtime statistics. Upon closer inspection, we see that robots moving in narrow corridors in opposite directions lead to many rainbow cycles. By enforcing a one-way policy in each corridor, similar to \cite{cohen2015feasibility}, many of these rainbow cycles can be eliminated and the performance of our method can be improved. Indeed, Figure~\ref{fig:warehouse_results} reports the results when paths are modified such that no horizontal corridor has robots moving in opposing directions.

We further use the same warehouse example to demonstrate how DrPP-RC can be used in conjunction with a higher-level emergency stopping algorithm. In practical examples, robots carry shelves around the warehouse, which might make it dangerous for humans to work in the same space. To guarantee safety for humans, we require robots to stop and give way if there is a human in a predefined radius. As the video in \url{https://youtu.be/gVSKs1iKsQw} shows, DrPP-RC guarantees that deadlocks and collisions are avoided in such cases.

\begin{figure}
	\centering
	\includegraphics[width=1\columnwidth, clip]{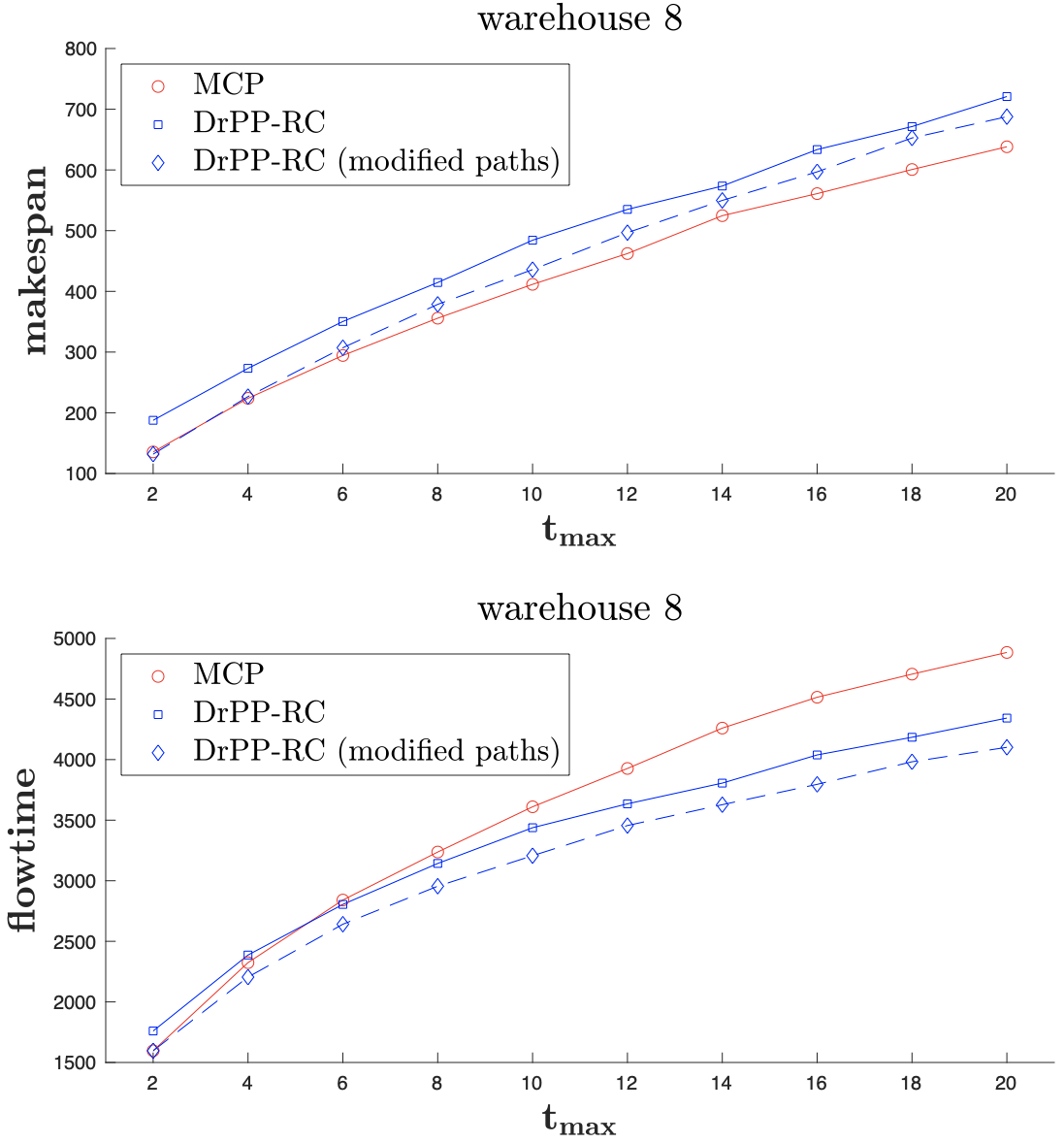}
	\caption{Makespan and flowtime statistics averaged over 1000 runs for the warehouse environment under varying $t_{max}$ values. Dashed lines show the improvement obtained by modifying the paths to decrease the number of rainbow cycles. DrPP-N cannot solve this instance as the collection of paths violate the conditions in Theorem \ref{theorem1}.}
	\label{fig:warehouse_results}
\end{figure}


\section{Conclusions}\label{chp:conclusion}
In this paper, we presented a method to solve the multi-robot path execution (MRPE) problem. Our method is based on a reformulation of the MRPE problem as an instance of drinking philosophers problem (DrPP). We showed that the existing solutions to the DrPP can be used to solve instances of MRPE problems if drinking sessions are constructed carefully. However, such an approach leads to conservative control policies. To improve the system performance, we provided a less conservative approach where we modified an existing DrPP solution. We provided conditions under which our control policies are shown to be collision and deadlock-free. We further demonstrated the efficacy of this method by comparing it with existing work. We observed that our method provides similar makespan performance to \cite{ma2017multi} while outperforming it in flowtime statistics, especially as uncertainty in robots' motion increase. This improvement can be explained mainly by our method's ability to change the priority order between robots during run-time, as opposed to keeping a fixed order. 

Our current method and derived conditions that guarantee collision and deadlock-freeness are limited to the multi-robot path execution problem where robot paths are assumed to be fixed a priori. Using such conditions to guarantee deadlock-freeness of replanning approaches or designing life-long planning algorithms with similar guarantees are interesting directions for future research. We are also interested in finding looser conditions that guarantee collision and deadlock-freeness, as the current conditions are sufficient but might not be necessary.

\begin{ack}
We thank Hang Ma from Simon Fraser University and Sven Koenig from University of Southern California for sharing their code for MCP implementation in \cite{ma2017multi} with us. We also thank Ruya Karagulle for pointing out typos in Theorem~1.  The last but not least, we thank the reviewers for their valuable comments and suggestions, which improved the clarity and the presentation of the paper greatly. This work is supported in part by ONR grant N00014-18-1-2501, NSF grant ECCS-1553873, and an Early Career Faculty grant from NASA's Space Technology Research Grants Program.
\end{ack}

\bibliographystyle{agsm}
\bibliography{references}

\appendix
\section*{Appendix - Proof of Theorem~\ref{theorem1} }

	We first start by showing that the Algorithm~\ref{alg:main} is collision-free. Assume that $r_n$ is currently occupying the shared cell $\pi_n^t$. Note that, when a robot is about to move to a shared cell, $GO$ action is issued only when $r_n$ is drinking (lines $20-23$). Therefore, before reaching $\pi_n^t$, $r_n$ was in drinking state, and thus, was holding all the bottles in $\tilde{\Session}_n(t)$. If $\pi_n^{t+1}$ is a free cell, $r_n$ would stay in drinking state until reaching $\pi_n^{t+1}$ (lines $28-35$). Otherwise, it would get insatiable with 
	$\Big(\Bottles_n\big(\tilde{\Session}_n(t)\big), \Bottles_n\big(\tilde{\Session}_n(t+1)\big)\Big)$.
	In neither of these scenarios, $r_n$ releases any bottles in $\Bottles_n(\pi_n^t) \subseteq \Bottles_n\big(\tilde{\mathcal{S}}_n(t)\big)$ before reaching to $\pi_n^{t+1}$. Note also that, by construction of drinking sessions, $\pi_n^t \subseteq \tilde{\mathcal{S}}_n(t)$. Since bottles are mutually exclusive, none of the other robots could acquire the bottles in $\Bottles_n(\pi_n^t)$ while $r_n$ is in $\pi_n^t$. This implies that collisions are avoided, as no other robot is allowed to occupy $\pi_n^t$ before $r_n$ leaves.
	
	We now show that Algorithm \ref{alg:main} is deadlock free. As defined in Definition~\ref{dfn:deadlock} deadlock is any configuration where a subset of robots, which have not reached their final cell, choose $STOP$ action indefinitely. As it can be seen from Algorithm~\ref{alg:main}, there are only three cases where a robot chooses the $STOP$ action: (i) when the robot is already in the final cell (line $2$), (ii) when there are no free cells from the next cell up to and including the final cell, and the final cell is not yet cleared by all other robots (line $14$), (iii) when the robot is in thirsty or insatiable state (line $19$). In the following, we show that none of these cases can cause a deadlock.
	
	We start by showing that neither (i) nor (ii) could cause a deadlock. 
	To do so, assume $r_n$ has reached its final cell and is causing a deadlock by blocking others from progressing. By (3) of Theorem~\ref{theorem1}, we know that there exist at least one free cell in each path. Since we assumed that $r_n$ is blocking others by waiting in its final cell, $\pi_n^{end}$ must be a shared cell. Then, there must be at least one free cell before $\pi_n^{end}$. Let $\pi_n^t$ denote the last free cell on $\pi_n$. A robot reaching a free cell gets into tranquil state, if it is not already in tranquil state, due to line $34$ of Algorithm~\ref{alg:main}. Otherwise, if $\pi_n^t$ is the first cell of $\pi_n$, $r_n$ would be in tranquil state before trying to move forward, since all robots are initialized in tranquil state. According to lines $12-15$ of Algorithm ~\ref{alg:main}, $r_n$ would wait in $\pi_n^t$ in tranquil state, until its final cell is cleared by all other robots. Since a tranquil robot does not need any bottles, no other robot could be waiting for $r_n$. However, this is a contradiction, and it is not possible for a robot to reach its final cell and block others from progressing. Therefore, (i) cannot be a reason for a deadlock. Furthermore, we showed that a robot waiting due to (ii) would stay in tranquil state until all others clear its final cell. As stated, a tranquil robot does not need any bottles, and thus, no other robot could be waiting for $r_n$. Thus, (ii) cannot cause deadlocks, either. 
	
	We now show that (iii) cannot cause deadlocks. To do so, assume that a subset of robots are stuck due to (iii), i.e., they are all in thirsty or insatiable state, and they need additional bottle(s) to move. If there was a robot who does not wait for any other robot, it would start drinking and moving. Therefore, some non-empty subset of these robots must be waiting circularly for each other. Without loss of generality, let $r_n$ be waiting for $r_{n+1}$ for $n \in \{1,\dots, K\}$ where $r_{K+1} = r_1$. That is, $r_{n}$ has some subset of bottles $r_{n-1}$ needs, and would not release them without acquiring some subset of bottles from $r_{n+1}$. Note that, there might be other robots choosing the $STOP$ action indefinitely as well, however, the main reason for the deadlock is this circular wait. Once the circular waiting ends, all robots would start moving according to their priority ordering. 
	
	For the time being, assume that each robot starts from a free inital cell and moves towards a free cell through an arbitrary number of shared cells in between. We later relax this assumption. Firstly, we know that none of the robots could be in tranquil or drinking state, otherwise they would be moving until reaching the next cell as lines $5-8$, $20-23$ and $28-31$ of Algorithm~\ref{alg:main}. Secondly, we show that, not all robots can be thirsty. Since a strict priority order is maintained between robots at all times, if all of them were thirsty, the robot with the highest priority would acquire all the bottles it needs according to $R'5$ and start drinking. A drinking robot starts moving, therefore cannot be participating in a deadlock. Therefore, there must be at least one robot that is in insatiable state. Thirdly, we show that if there is a deadlock, all robots participating in it must be in insatiable state. To show a contradiction, assume that at least one of the robots participating in the deadlock is thirsty. According to $R'5$, an insatiable robot always has a higher priority than a thirsty robot. Therefore, an insatiable robot cannot be waiting for a thirsty robot. Thus, all robots in a deadlock configuration must in insatiable state.

Let $\tilde{G}_\Pi$ be the graph returned by the Algorithm~\ref{alg:cycles} for the input \graphname~$G_\Pi$. We showed that all robots are in insatiable state. Let $\pi_n^{t_n}$ denote the current cell $r_n$ is occupying, i.e., $curr(r_n) = t_n$. That is, $r_n$ is insatiable with $\Big(\Bottles_n\big(\tilde{\Session}_n(t_n)\big), \Bottles_n\big(\tilde{\Session}_n(t_n+1)\big)\Big)$ and needs all the bottles in $\Bottles_n(\tilde{\Session}_n(t_n)\cup \tilde{\Session}_n(t_n+1))$ to start drinking. Since $r_n$ currently occupies $\pi_n^{t_n}$, it holds all the bottles in $\Bottles_n(\tilde{\Session}_n(t_n))$. This implies that $\tilde{\Session}_n(t_n)\not=\tilde{\Session}_n(t_n+1)$, and that $r_n$  does not hold some of the bottles in $\Bottles_n(\tilde{\Session}_n(t_n+1))$. Now define $[\tilde{\mathcal{S}}_n(t_n)] \doteq \bigcup_{v_i \in \tilde{\mathcal{S}}_n(t_n)} \{[v_i] \}$. By construction, there must be two nodes in $\tilde{G}_\Pi$, one corresponding to $[\tilde{\Session}_n(t_n)]$ and another corresponding to $[\tilde{\Session}_n(t_n+1)]$, and a $c_n$ colored edge from $[\tilde{\Session}_n(t_n)]$ to $[\tilde{\Session}_n(t_n+1)]$ in $\tilde{G}_\Pi$. Similarly, $r_{n+1}$ holds all the bottles in $\Bottles_{n+1}(\tilde{\Session}_{n+1}(t_{n+1}))$ and is missing some of the bottles in $\Bottles_{n+1}(\tilde{\Session}_{n+1}(t_{n+1}+1))$. Since $r_n$ is waiting for $r_{n+1}$, either $[\tilde{\Session}_n(t_n+1)] = [\tilde{\Session}_{n+1}(t_{n+1})]$ or $[\tilde{\Session}_n(t_{n}+1)] = [\tilde{\Session}_{n+1}(t_{n+1}+1)]$ must hold. This implies that, there exists a $c_{n}$ colored edge from $[\tilde{\Session}_{n}(t_{n})]$ to either $[\tilde{\Session}_{n+1}(t_{n+1})]$ or to $[\tilde{\Session}_{n+1}(t_{n+1}+1)]$. In a similar manner, either $[\tilde{\Session}_{i}(t_{i})] = [\tilde{\Session}_{i+1}(t_{i+1})]$ or $[\tilde{\Session}_{i}(t_{i})] = [\tilde{\Session}_{i+1}(t_{i+1}+1)]$ holds for each $i \in \{1,\dots, K\}$.

Assume $[\tilde{\Session}_n(t_n+1)] = [\tilde{\Session}_{n+1}(t_{n+1})]$ holds for each $n\in \{1, \dots, K\}$. Then, there exists a rainbow cycle $\{ ([\tilde{\Session}_1(t_1)],$ $ c_1, [\tilde{\Session}_2(t_2)]), \dots,$ $ ([\tilde{\Session}_K(t_K)], c_N,[\tilde{\Session}_1(t_1)])\}$ in $\tilde{G}_\Pi$. However, this is a contradiction since $\tilde{G}_\Pi$ must be free of rainbow cycles by construction and such a rainbow cycle would be found by the Algorithm \ref{alg:cycles}. 

Alternatively, let $m\in \{1, \dots, K\}$ be arbitrary and $[\tilde{\Session}_m(t_m+1)] = [\tilde{\Session}_{m+1}(t_{m+1}+1)]$ be true. For each $n\not= m$, assume $[\tilde{\Session}_n(t_n+1)] = [\tilde{\Session}_{n+1}(t_{n+1})]$ holds. Then, a rainbow cycle can be created similar to the previous case, where the edges $([\tilde{\Session}_m(t_m)], c_m,[\tilde{\Session}_{m+1}(t_{m+1})])$ and $([\tilde{\Session}_{m+1}(t_{m+1})], c_{m+1},$ $[\tilde{\Session}_{m+2}(t_{m+2})])$ are replaced by $([\tilde{\Session}_m(t_m)],$ $c_m,[\tilde{\Session}_{m+2}(t_{m+2})])$. Again, this is a contradiction because $\tilde{G}_\Pi$ must be free of rainbow cycles. 

Similarly, we can construct a rainbow cycle when $[\tilde{\Session}_m(t_m+1)] = [\tilde{\Session}_{m+1}(t_{m+1}+1)]$ holds for more than one robot. The length of the rainbow cycle would decrease by the number of robots for which the previous condition holds. On the extreme case, assume $[\tilde{\Session}_n(t_n+1)] = [\tilde{\Session}_{n+1}(t_{n+1}+1)]$ for all $n\in \{1, \dots, K\}$. In this case, we can no longer construct a rainbow cycle as before. However, this case implies all robots are insatiable with $\big(\Bottles_n({\Session}_{n}(t_n)), \Bottles_{n}({\Session}_{n}(t_n+1))\big)$ where $[{\Session}_{n}(t_n+1) ] = [{\Session}_{m}(t_m +1)]$ for each pair of $(m, n)$. In such a case, the robot $r_p$ with the highest priority would obtain all the bottles in $\Bottles_{p}({\Session}_{p_2})$ due to $\mathbf{R'5}(3a)$, and start drinking. After $r_p$ no longer needs bottles in $[{\Session}_{p}(t_p+1) ]$, robot with the second highest priority would start drinking. This pattern would be repeated until there are no robots waiting. This is a contradiction that there was a deadlock. Therefore, we showed that  a deadlock configuration cannot be reached and (iii) cannot be a reason for a deadlock. 
	
	We now relax the assumption that all robots are initialized at a free cell. To do so, we \emph{``modify"} all paths by appending \emph{virtual} free cell at the beginning. That is, all robots are initialized at a virtual free cell, which does not exist physically, and the next cell in a robot's path is its original initial cell. Theorem~\ref{theorem1} assumes that initial drinking sessions are disjoint for each robot, i.e., $\mathcal{S}_m(0) \cap \mathcal{S}_n(0) = \emptyset$ for all $m\not=n$. Since initial drinking sessions are disjoint, all robots whose initial cell is a shared cell can immediately start drinking. As a result, all of those robots can immediately \emph{``virtually move"} into their original initial cell. All other robots with free initial cells can also move to their original initial cells immediately. Therefore the assumption that all robots are initialized at a free cell is not restricting. 
	
	Finally, we relax the assumption that each robot moves towards a free cell. Theorem~\ref{theorem1} requires each path to have at least one free cell. Then, up until reaching the final free cell, moving towards a free cell assumption is not restrictive. We know under this condition that deadlocks are prevented, therefore all robots are at least guaranteed to reach to the final free cell in their path. Theorem~\ref{theorem1} also requires that the final drinking sessions are disjoint. Therefore, all robots would eventually be able to start drinking and reach their final location.

	Deadlocks occur when a subset of robots, which have not reached their final cell, choose $STOP$ action indefinitely. A robot chooses the $STOP$ action only under three conditions. We showed that none of these conditions can cause a deadlock. Thus, Algorithm \ref{alg:main} is deadlock-free. 
	\hfill$\square$

\end{document}